\documentclass{article}

\usepackage{microtype}
\usepackage{graphicx}
\usepackage{subfigure}
\usepackage{booktabs} 

\usepackage{amsfonts}       
\usepackage{nicefrac}       
\usepackage{microtype}      

\usepackage{amsmath}
\usepackage{amsthm}
\usepackage{amssymb}
\usepackage{bm}
\usepackage{mathtools}

\newcommand{\specialcell}[2][c]{%
  \begin{tabular}[#1]{@{}c@{}}#2\end{tabular}}
\newcommand{\latin}[1]{{\it #1}}

\usepackage{color}
\newcounter{nbdrafts}
\makeatletter
\newcommand{\checknbdrafts}{
\ifnum \thenbdrafts > 0
\@latex@warning@no@line{**********************************************************************}
\@latex@warning@no@line{* The document contains \thenbdrafts \space draft note(s)}
\@latex@warning@no@line{**********************************************************************}
\fi}

\makeatother

\usepackage{hyperref}

\usepackage[accepted]{icml2018}

\icmltitlerunning{Knowledge Transfer with Jacobian Matching}

\newtheorem{prop}{Proposition}
\newtheorem*{claim}{Claim}

\newtheorem{prop_star}{Proposition}

\newtheorem*{corollary}{Corollary}

\newcommand{\X}{\mathbf{x}}
\newcommand{\Y}{\mathbf{y}}

\newcommand{\XI}{\boldsymbol{\xi}}

\begin{document}

\twocolumn[
\icmltitle{Knowledge Transfer with Jacobian Matching}



\icmlsetsymbol{equal}{*}

\begin{icmlauthorlist}
\icmlauthor{Suraj Srinivas}{id}
\icmlauthor{Fran\c{c}ois Fleuret}{id}
\end{icmlauthorlist}

\icmlaffiliation{id}{Idiap Research Institute \& EPFL, Switzerland}

\icmlcorrespondingauthor{Suraj Srinivas}{suraj.srinivas@idiap.ch}

\icmlkeywords{Machine Learning, ICML}

\vskip 0.3in
]



\printAffiliationsAndNotice{}  

\begin{abstract}
Classical distillation methods transfer representations from a ``teacher'' neural network to a ``student'' network by matching their output activations. Recent methods also match the Jacobians, or the gradient of output activations with the input. However, this involves making some ad hoc decisions, in particular, the choice of the loss function.

In this paper, we first establish an equivalence between Jacobian matching and distillation with input noise, from which we derive appropriate loss functions for Jacobian matching. We then rely on this analysis to apply Jacobian matching to transfer learning by establishing equivalence of a recent transfer learning procedure to distillation.

We then show experimentally on standard image datasets that Jacobian-based penalties improve distillation, robustness to noisy inputs, and transfer learning.
\end{abstract}

\section{Introduction}
Consider that we are given a neural network $\mathcal{A}$ trained on a particular dataset, and want to train another neural network $\mathcal{B}$ on a similar (or related) dataset. Is it possible to leverage $\mathcal{A}$ to train $\mathcal{B}$ more efficiently? We call this the problem of \latin{knowledge transfer}. Distillation~\citep{hinton2015distilling} is a form of knowledge transfer where $\mathcal{A}$ and $\mathcal{B}$ are trained on the same dataset, but have different architectures. Transfer Learning~\citep{pan2010survey} is another form of knowledge transfer where $\mathcal{A}$ and $\mathcal{B}$ are trained on different (but related) datasets. If the architectures are the same, we can in both cases simply copy the weights from $\mathcal{A}$ to $\mathcal{B}$.  The problem becomes more challenging when $\mathcal{A}$ and $\mathcal{B}$ have different architectures.

A perfect distillation method would enable us to transform one neural network architecture into another, while preserving generalization. This would allow us to easily explore the space of neural network architectures, which can be used for neural network architecture search, model compression, or creating diverse ensembles. A perfect transfer learning method, on the other hand, would use little data to train $\mathcal{B}$, optimally using the limited samples at it's disposal.

This paper deals with the problem of knowledge transfer based on matching the Jacobian of the network's output which, like the output itself, has a dimension independent of the network's architecture. This approach has also been recently explored for the case of distillation by~\citet{czarnecki2017sobolev}, who considered the general idea of matching Jacobians, and by~\citet{zagoruyko2016paying} who view Jacobians as attention maps. However, in both of these cases, it was not clear what loss function must be used to match Jacobians. It was also unclear how these methods relate to classical distillation approaches~\citep{ba2014deep, hinton2015distilling}.

Recently~\citet{li2016learning} proposed a distillation-like approach to perform transfer learning. While this approach works well in practice, it was not clear how this exactly relates to regular distillation. It is also not clear how this applies to the challenging setting of transfer learning where the architectures of both networks $\mathcal{A}$ and $\mathcal{B}$  can be arbitrary.

The overall contributions of our paper are: 
\begin{enumerate}
\item We show that matching Jacobians is a special case of classical distillation, where noise is added to the inputs. 
\item We show that a recent transfer learning method (LwF by \citealp{li2016learning}) can be viewed as distillation, which allows us to match Jacobians for this case.
\item We provide methods to match Jacobians of practical deep networks, where architecture of both networks are arbitrary.
\end{enumerate}

We experimentally validate these results by providing evidence that Jacobian matching helps both regular distillation and transfer learning, and that Jacobian-norm penalties learn models robust to noise.

\section{Related Work}
Several Jacobian-based regularizers have been proposed in recent times. Sobolev training~\citep{czarnecki2017sobolev}, showed that using higher order derivatives along with the targets can help in training with less data. This work is similar to ours. While we also make similar claims, we clarify the relationship of this method with regular distillation based on matching activations, and show how it can help. Specifically, we show how the loss function used for activation matching also decides the loss function we use for Jacobian matching. Similarly, \citet{wang2016analysis} use the Jacobian for distillation and show that it helps improve performance. \citet{zagoruyko2016paying} introduce the idea of matching attention maps. The Jacobian was also considered as one such attention map. This work also finds that combining both activation matching and Jacobian matching is helpful. 

Jacobian-norm regularizers were used in early works by~\citet{drucker1992improving}, where they looked at penalizing the Jacobian norm. The intuition was to make the model more robust to small changes in the input. We find that this conforms to our analysis as well.

Knowledge Distillation~\citep{hinton2015distilling} first showed that one can use softmax with temperature to perform knowledge transfer with neural nets.~\citet{ba2014deep} found that squared error between logits worked better than the softmax method, and they used this method to train shallow nets with equivalent performance to deep nets.~\citet{romero2014fitnets} and~\citet{zagoruyko2016paying} showed how to enhance distillation by matching intermediate features along with the outputs, but use different methods to do so.~\citet{sau2016deep} found that adding noise to logits helps during teacher-student training. We show that the use of the Jacobian can be interpreted as adding such noise to the inputs analytically.

\section{Jacobians of Neural Networks}

Let us consider the first order Taylor series expansion of a function $f: \mathbb{R}^D \rightarrow \mathbb{R}$ around a small neighborhood $\{\X + \Delta \X : \|\Delta \X \| \leq \epsilon \}$. It can be written as
\begin{align}
\label{defn}
f(\X + \Delta \X) &= f(\X) + \nabla_x f(\X)^T (\Delta \X) + \mathcal{O}(\epsilon^2) 
\end{align}
We can apply this linearization to neural nets. The source of non-linearity for neural nets lie in the elementwise non-linear activations (like ReLU, sigmoid) and pooling operators. \textit{It is easy to see that to linearize the entire neural network, one must linearize all such non-linearities in the network.}

\subsection{Special case: ReLU and MaxPool} For the ReLU nonlinearity, the Taylor approximation is locally exact and simple to compute, as the derivative $\frac{\mathrm{d}\sigma(z)}{\mathrm{d}z}$ is either 0 or 1 (except at $z = 0$, where it is undefined). A similar statement holds for max-pooling. Going back to the definition in Equation \ref{defn}, for piecewise linear nets there exist $\epsilon > 0$ such that the super-linear terms are exactly zero, \latin{i.e.}; $f(\X + \Delta \X) = f(\X) + \nabla_x f(\X)^T (\Delta \X)$, \latin{i.e.}; $\X$ and $\Delta \X$ lie on the same linear surface.

\subsection{Invariance to weight and architecture specification} One useful property of the Jacobian is that its size does not depend on the network architecture. For $k$ output classes, and input dimension $D$ , the Jacobian of a neural network is of dimension $D \times k$. This means that one can compare Jacobians of different architectures.

Another useful property is that for a given neural network architecture, different weight configurations can lead to the same Jacobian. One simple example of this is permutation symmetry of neurons in intermediate hidden layers. It is easy to see different permutations of neurons leave the Jacobian unchanged (as they have the same underlying function). In general, because of redundancy of neural network models and non-convexity of the loss surface, several different weight configurations can end up having similar Jacobians.  

Given these desirable properties, we consider using the Jacobian to perform knowledge transfer between neural networks of different architectures. Note that these properties hold trivially for output activations also. Thus it seems sensible that both these quantities must be used for knowledge transfer. However, the important practical question remains: how exactly should this be done?

\section{Distillation}

We consider the problem of improving distillation using Jacobians. This problem of distillation may be posed as follows: given a \textit{teacher} network $\mathcal{T}$ which is trained on a dataset $\mathcal{D}$, we wish to enhance the training of a student network $\mathcal{S}$ on $\mathcal{D}$ using ``hints'' from $\mathcal{T}$. Classically, such ``hints'' involve activations of the output layer or some intermediate layers. Recent works~\citep{czarnecki2017sobolev, zagoruyko2016paying} sought to match the Jacobians of $\mathcal{S}$ and $\mathcal{T}$. However, two aspects are not clear in these formalisms: (i) what penalty term must be used between Jacobians, and (ii) how this idea of matching Jacobians relates to simpler methods such as activation matching~\citep{ba2014deep, hinton2015distilling}. To resolve these issues, we make the following claim.

\begin{claim}
	Matching Jacobians of two networks is equivalent to matching soft targets with noise added to the inputs during training.
\end{claim}

\begin{figure}
\centering
\includegraphics[width=8.25cm]{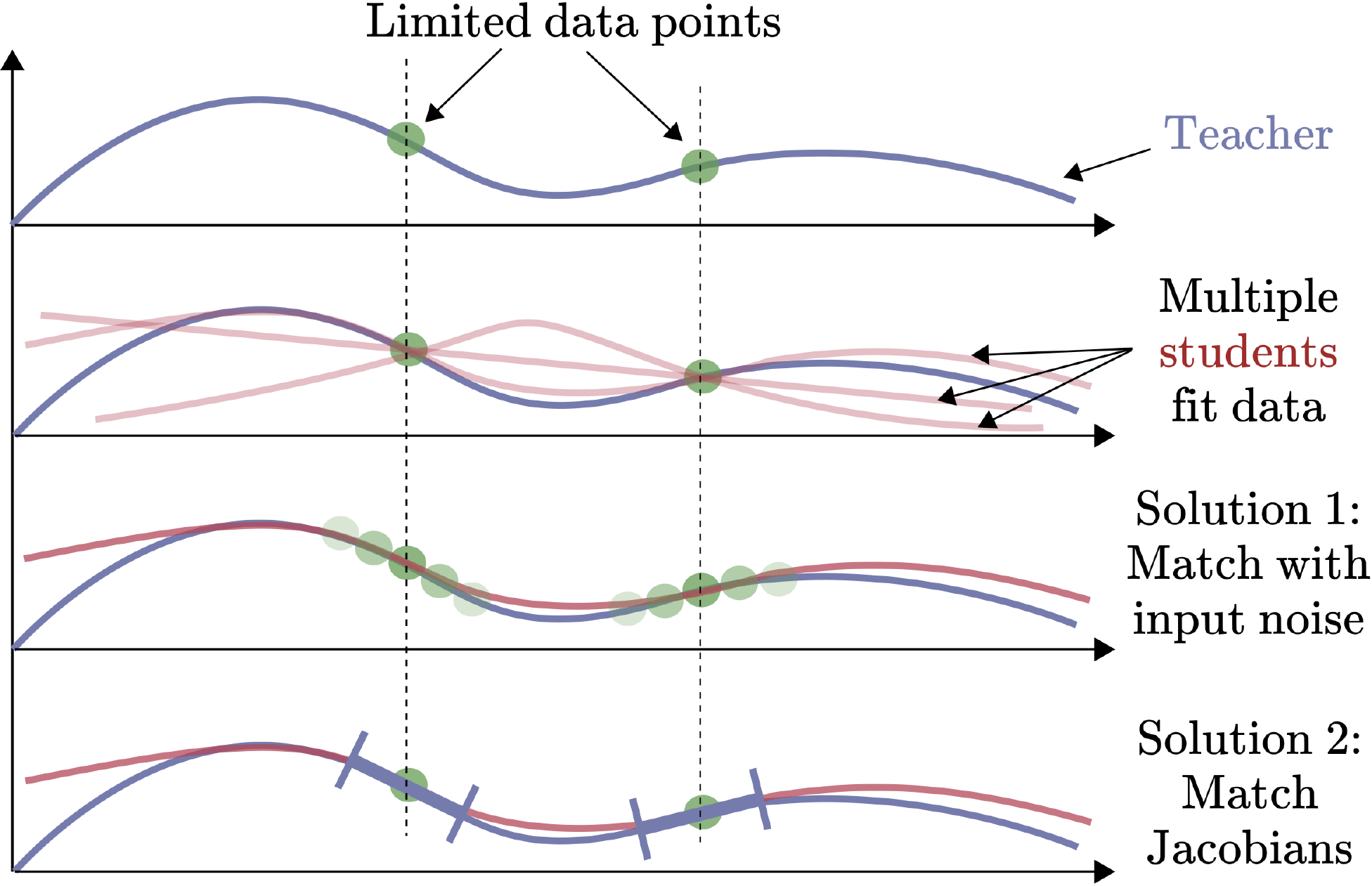}
\caption{Illustration of teacher-student learning in a simple 1D case. Here, x-axis is the input data, and y-axis denotes function outputs. Given a limited number of data points, there exist multiple student functions consistent with the data. How do we select the hypothesis closest to the teacher's? There are two equivalent solutions: either by augmenting the data set by adding noise to the inputs or by directly matching slopes (Jacobians) of the function at the data points.}
\label{fig:noise}
\end{figure}

More concretely, we make the following proposition.

\begin{prop}
\label{prop1}
	Consider the squared error cost function for matching soft targets of two neural networks with $k$-length targets ($ \in \mathbb{R}^k$), given by $\ell(\mathcal{T}(\X), \mathcal{S}(\X)) = \sum_{i=1}^{k}(\mathcal{T}^i(\X) - \mathcal{S}^i(\X))^2$, where $\X \in \mathbb{R}^D$ is an input data point. Let $\XI ~(\in \mathbb{R}^D)= \sigma~\mathbf{z}$ be a scaled version of a unit normal random variable $\mathbf{z}~ \in \mathbb{R}^D$ with scaling factor $\sigma \in \mathbb{R}$. Then the following is true.

\vspace*{-20pt}	
	\begin{align*}
	&\mathbb{E}_{\XI}\left[\sum_{i=1}^{k}\left(\mathcal{T}^i(\X + \XI) - \mathcal{S}^i(\X + \XI)\right)^2\right] \\ & = \sum_{i=1}^{k}\left(\mathcal{T}^i(\X) - \mathcal{S}^i(\X)\right)^2 + \sigma^2 \sum_{i=1}^{k} \|\nabla_x\mathcal{T}^i(\X) - \nabla_x\mathcal{S}^i(\X)\|_2^2 \\&+ \mathcal{O}(\sigma^4)
	\end{align*}

\end{prop}

Notice that in this expression, we have decomposed the loss function into two components: one representing the usual distillation loss on the samples, and the second regularizer term representing the Jacobian matching loss. The higher order error terms are small for small $\sigma$ and can be ignored. The above proposition is a simple consequence of using the first-order Taylor series expansion around $x$. Note that the error term is exactly zero for piecewise-linear nets. An analogous statement is true for the case of cross entropy error between soft targets, leading to:
\vspace*{-5pt}
\begin{align}
\label{eqn:TSCE}
&\mathbb{E}_{\XI}\left[-\sum_{i=1}^{k} \mathcal{T}^i_s(\X + \XI) \log\left(\mathcal{S}^i_s(\X + \XI)\right)\right] \\ & \approx  -\sum_{i=1}^{k} \mathcal{T}^i_s(\X) \log(\mathcal{S}^i_s(\X)) ~-~ \sigma^2 \sum_{i=1}^{k} \frac{\nabla_x\mathcal{T}^i_s(\X)^T \nabla_x\mathcal{S}^i_s(\X)}{\mathcal{S}^i_s(\X)} \nonumber
\end{align}

where $\mathcal{T}^i_s(\X)$ denotes the same network $\mathcal{T}^i(\X)$ but with a softmax or sigmoid (with temperature parameter $T$ if needed) added at the end. We do not write the super-linear error terms for convenience. This shows that the Jacobian matching loss depends crucially on the loss used to match activations. This observation can be used in practice to pick appropriate loss function by choosing a specific noise model of interest.

To summarize, these statements show that matching Jacobians is a natural consequence of matching not only the raw CNN outputs at the given data points, but also at the infinitely many data points nearby. This is illustrated in Figure \ref{fig:noise}, which shows that by matching on a noise-augmented dataset, the student is able to mimic the teacher better.

We can use a similar idea to derive regularizers for the case of regular neural network training as well. These regularizers seek to make the underlying model \textit{robust} to small amounts of noise added to the inputs.

\begin{prop}
\label{prop2}
	Consider the squared error cost function  for training a neural network with $k$ targets, given by $\ell(y(\X), \mathcal{S}(\X)) = \sum_{i=1}^{k}(y^i(\X) - \mathcal{S}^i(\X))^2$, where $\X \in \mathbb{R}^D$ is an input data point, and $y^i(\X)$ is the $i^{th}$ target output. Let $\XI ~(\in \mathbb{R}^D)= \sigma~\mathbf{z}$ be a scaled version of a unit normal random variable $\mathbf{z}~ \in \mathbb{R}^D$ with scaling factor $\sigma \in \mathbb{R}$. Then the following is true.
\vspace*{-2pt}
	\begin{align*}
	&\mathbb{E}_{\XI}\left[\sum_{i=1}^{k} \left(y^i(\X) - \mathcal{S}^i(\X + \XI)\right)^2\right] \\ & = \sum_{i=1}^{k}\left(y^i(\X) - \mathcal{S}^i(\X)\right)^2 + \sigma^2 \sum_{i=1}^{k} \|\nabla_x\mathcal{S}^i(\X)\|_2^2 + \mathcal{O}(\sigma^4)
	\end{align*}

\end{prop}

A statement similar to Proposition \ref{prop2} has been previously derived by~\citet{bishop1995training}, who observed that the regularizer term for linear models corresponds exactly to the well-known Tikhonov regularizer. This regularizer was also proposed by~\citet{drucker1992improving}. The $\ell_2$ weight decay regularizer for neural networks can be derived by applying this regularizer layer-wise separately. However, we see here that a more appropriate way to ensure noise robustness is to penalize the norm of the Jacobian rather than weights. We can derive a similar result for the case of cross-entropy error as well, which is given by -

\begin{align}
\label{eqn:RegCE}
&\mathbb{E}_{\XI}\left[- \sum_{i=1}^{k} y^i(\X) \log(\mathcal{S}^i_s(\X + \XI))\right] \\ & \approx  - \sum_{i=1}^{k} y^i(\X) \log(\mathcal{S}^i_s(\X)) ~+~ \sigma^2 \sum_{i=1}^{k} y^i(\X) \frac{\|\nabla_x\mathcal{S}^i_s(\X)\|_2^2}{\mathcal{S}^i_s(\X)^2} \nonumber
\end{align}

We notice here again that the regularizer involves $\mathcal{S}^i_s(\X)$, which has the sigmoid / softmax nonlinearity applied on top of the final layer of $\mathcal{S}^i(\X)$. Deriving all the above results is a simple matter of using first-order Taylor series expansions, and a second-order expansion in the case of Equation \ref{eqn:RegCE}. Proof is provided in the supplementary material.

In all cases above we see that the regularizers for cross-entropy error term seem more unstable when compared to those for squared error. We find that this is true experimentally as well. As a result, we use squared error loss for distillation.

\subsection{Approximating the Full Jacobian}
One can see that both in the case of Proposition \ref{prop1} and \ref{prop2}, we are required to compute the full Jacobian. This is computationally expensive, and sometimes unnecessary. For example, Equation \ref{eqn:RegCE} requires only the terms where $y^i(\X)$ is non-zero.

In general, we can approximate the summation of Jacobian terms with the one with largest magnitude. However, we cannot estimate this without computing the Jacobians themselves. As a result, we use a heuristic where the only output variable involving the correct answer $c \in [1,k]$ is used for computing the Jacobian. This corresponds to the case of Equation \ref{eqn:RegCE}. Alternately, if we do not want to use the labels, we may instead use the output variable with the largest magnitude, as it often corresponds to the right label (for good models). 

\section{Transfer Learning} 

We now use our Jacobian matching machinery to transfer learning problems. In computer vision, transfer learning is often done by fine-tuning~\citep{yosinski2014transferable}, where models pre-trained on a large dataset $\mathcal{D}_l$, such as Imagenet~\citep{russakovsky2015imagenet}, are used as initialization for training on another smaller dataset $\mathcal{D}_s$. We can also refer to these as the source dataset and the target dataset respectively. Practically, this means that the architecture used for fine-tuning must be the same as that of the pre-trained network, which is restrictive. We would like to develop transfer learning methods where the architectures of the pre-trained network and target ``fine-tuned'' network can be arbitrarily different. 

\begin{figure}
\centering
\includegraphics[width=8.5cm]{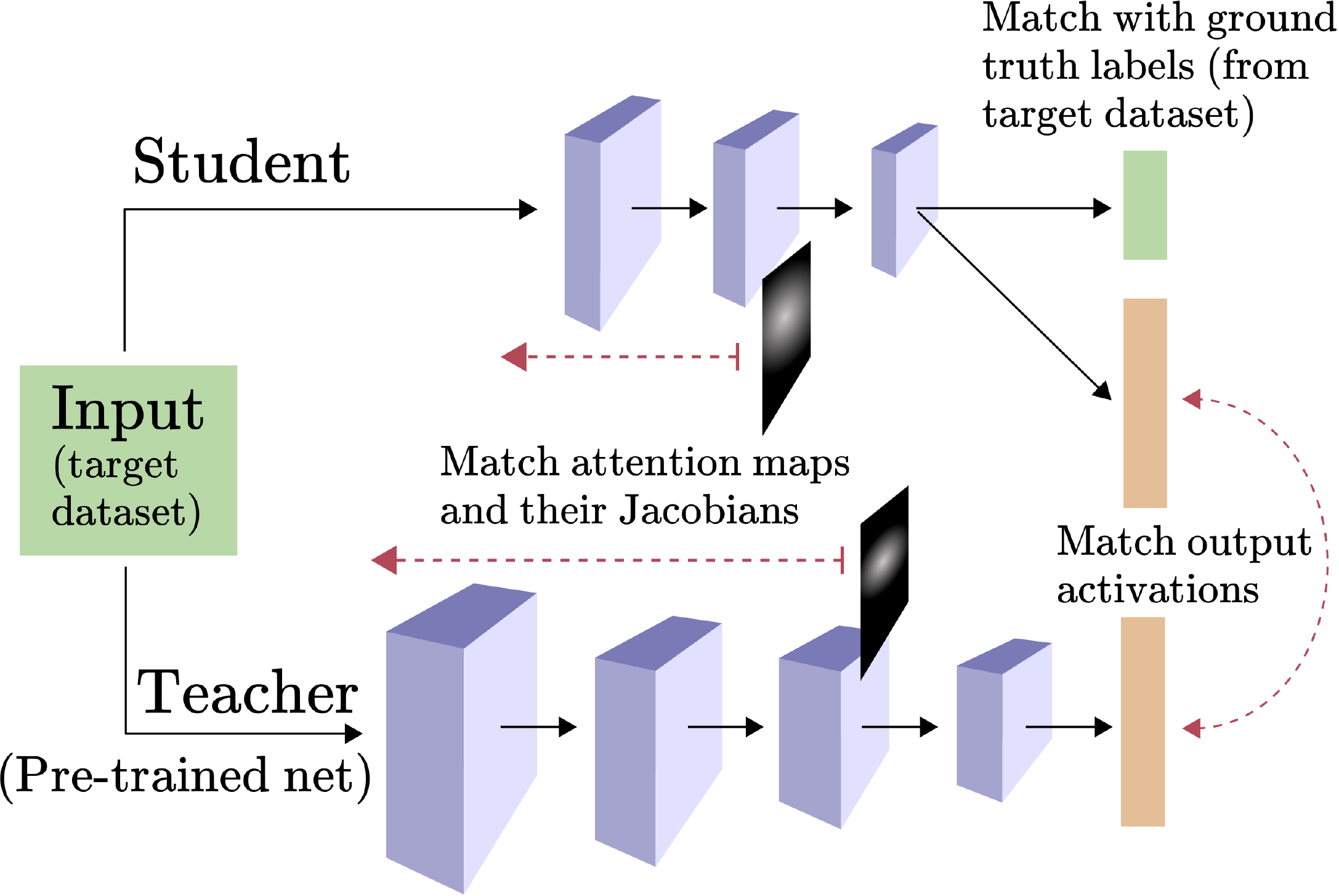}
\caption{Illustration of our proposed method for transfer learning. We match the output activations of a pre-trained Imagenet network similar to LwF~\citep{li2016learning}. We also match aggregated activations or ``attention'' maps between networks, similar to the work of~\citet{zagoruyko2016paying}. We propose to match Jacobians of (aggregated) attention maps w.r.t. inputs. }
\label{fig:block_diagram}
\end{figure}

One way to achieve this is by distillation: we match output activations of a pre-trained teacher network and an untrained student network. However, this procedure is not general as the target dataset may not share the same label space as the large source dataset. This leads to size mismatch between output activations of the teacher and the student. To overcome this, we can design the student network to have two sets of outputs (or two output ``branches''), one with the label space of the smaller target dataset, while the other with that of the larger source dataset. This leads to the method proposed by~\citet{li2016learning}, called ``Learning without Forgetting'' (\textbf{LwF}). Note that similar methods were concurrently developed by~\citet{jung2016less} and~\citet{furlanello2016active}. In this method, the student network is trained with a composite loss function involving two terms, one in each output branch. The two objectives are \textbf{(1)} matching ground truth labels on the target dataset, and \textbf{(2)} matching the activations of the student network and a pre-trained teacher network on the target dataset.  This is illustrated in Figure \ref{fig:block_diagram}. Crucially, these losses are matched only on the target dataset, and the source data is untouched. This is conceptually different from distillation, where the teacher network is trained on the dataset being used. In LwF, the pre-trained teacher is not trained on the target dataset.

This makes it problematic to apply our Jacobian matching framework to LwF. For distillation, it is clear that adding input noise (or Jacobian matching) can improve overall matching as shown in Figure \ref{fig:noise}. For the case of LwF, it is not clear whether improving matching between teacher and student will necessarily improve transfer learning performance. This is especially because the teacher is not trained on the target dataset, and can potentially produce noisy or incorrect results on this data. To resolve this ambiguity, we shall now connect LwF with distillation.

\subsection{LwF as Distillation} 

In the discussion below we shall only consider the distillation-like loss of LwF, and ignore the branch which matches ground truth labels. For LwF to work well, it must be the case that the activations of the pre-trained teacher network on the target dataset must contain information about the source dataset (\latin{i.e.}; Imagenet). The attractiveness of LwF lies in the fact that this is done without explicitly using Imagenet. Here, we make the claim that \latin{LwF approximates distillation on (a part of) Imagenet.}

Let $f(\cdot)$ be an untrained neural network, $g(\cdot)$ be a pre-trained network, $\X,\Y$ be the input image and corresponding ground truth label respectively. Let $|\mathcal{D}|$ be the size of the dataset $\mathcal{D}$. Let us denote $\rho(\X) = \ell(f(\X), g(\X))$ for convenience. Assume Lipschitz continuity for $\rho(\X)$ with Lipschitz constant $\mathrm{K}$, and distance metric $\psi_{\X}$ in the input space
\begin{align}
\| \rho(\X_1) - \rho(\X_2) \| \le \mathrm{K} \psi_{\X}(\X_1, \X_2) 
\end{align}
Note here that the distance in the input space need not be in terms of pixelwise distances, but can also be neural net feature distance, for example. Let us also define an assymetric version of the Hausdorff distance between two sets $A, B$:
 \vspace*{-10pt}
\begin{align}
\mathcal{H}_a(A, B) = \sup_{a \in A} \inf_{b \in B} \psi_{\X}(a,b).
\end{align}
%
Note that this is no longer a valid distance metric unlike the Hausdorff. Given these assumptions, we are now ready to state our result.  

\begin{prop}
\label{eqn:prop3}
Given the assumptions and notations described above, we have
\begin{align}
\frac{1}{|\mathcal{D}_l|} \sum_{\X \sim \mathcal{D}_l} \ell (f(\X), g(\X)) &\le  \max_{\X \sim \mathcal{D}_s} \ell (f(\X), g(\X)) \label{term:max_loss} \\ &+ \mathrm{K} \mathcal{H}_a(\mathcal{D}_l, \mathcal{D}_s) \label{term:distance}
\end{align}

\end{prop}

On the left side of \ref{term:max_loss} we have the distillation loss on the source dataset, and on the right we have a max-loss term on the target dataset. Note that the LwF loss is an average loss on the target dataset. As expected, the slack terms in the inequality depends on the distance between the source and target datasets (\ref{term:distance}). This bounds a loss related to the LwF loss (\latin{i.e.} max-loss instead of average) with the distillation loss. If the Hausdorff distance is small, then reducing the max-loss would reduce the distillation loss as well. A similar theory was previously presented by \citet{ben2010theory}, but with different formalisms. Our formalism allows us to connect with Jacobian matching, which is our primary objective.

In practice, the target dataset is often much smaller than the Imagenet and has different overall statistics. For example, the target dataset could be a restricted dataset of flower images. In such a case, we can restrict the source dataset to its ``best'' subset, in particular with all the irrelevant samples (those far from target dataset) removed. This would make the Hausdorff distance smaller, and provide a tighter bound. In our example, this involves keeping only flowers from Imagenet.

This makes intuitive sense as well: if the source and target datasets are completely different, we do not expect transfer learning (and thus LwF) to help. The greater the overlap between source and target datasets, the smaller is the Hausdorff distance, the tighter is the bound, and the more we expect knowledge transfer to help. Our results capture this intuition in a rigorous manner. In addition, this predicts that Lipschitz-smooth teacher neural nets that produce small feature distance between source and target images are expected to do well in transfer learning. Lipschitz-smoothness of models has been previously related to adversarial noise robustness~\citep{cisse2017parseval}, and to learning theory as a sufficient condition for generalization~\citep{xu2012robustness}. It is interesting that this relates to transfer learning as well.  

More importantly, this establishes LwF as a distillation method. The following result motivates input noise added to the target dataset.

\begin{corollary}
For any superset $\mathrm{D'}_s \supseteq \mathcal{D}_s$ of the target dataset, $ \mathcal{H}_a(\mathcal{D}_l, \mathcal{D'}_s) \leq \mathcal{H}_a(\mathcal{D}_l, \mathcal{D}_s)$
\end{corollary}

Thus if we augment the target dataset $\mathcal{D}_s$ by adding noise, we expect the bound in Proposition \ref{eqn:prop3} to get tighter. This is because when we add noise to points in $\mathcal{D}_s$, the minimum distance between points from $\mathcal{D}_l$ to $\mathcal{D}_s$ decreases. Proofs are provided in the supplementary material. 

To summarize, we have showed that a loss related to the LwF loss (max-loss) is an upper bound on the true distillation loss. Thus by minimizing the upper bound, we can expect to reduce the distillation loss also.

\subsubsection{Incorporating Jacobian matching} Now that input noise and thus Jacobian matching is well motivated, we can incorporate these losses into LwF. When we implemented this for practical deep networks we found that the optimizer wasn't able to reduce the Jacobian loss at all. We conjecture that it might be because of a vanishing gradient effect / network degeneracy on propagation of second order gradients through the network (and not the first). As a result, we need an alternative way to match Jacobians.

\subsection{Matching attention maps}

It is often insufficient to match only output activations between a teacher and a student network, especially when both networks are deep. In such cases we can consider matching intermediate feature maps as well. In general it is not possible to match the full feature maps between an arbitrary teacher and student network as they may have different architectures, and features sizes may never match at any layer. However, for modern convolutional architectures, spatial sizes of certain features often match across architectures even when the number of channels doesn't.~\citet{zagoruyko2016paying} noticed that it in such cases it helps to match channelwise aggregated activations, which they call \latin{attention} maps. Specifically, this aggregation is performed by summing over squared absolute value of channels of a feature activation $Z$, and is given by -

\vspace*{-10pt}
\begin{equation}
\label{eqn:attention}
A = AttMap(Z) = \sum_{\mathclap{{i \in channels}}} |Z_i|^2
\end{equation}

Further, the loss function used to match these activations is 

\begin{equation}
\label{eqn:matchact}
\mathrm{Match~Activations} = \left|\left| \frac{A_t}{\|A_t \|_2} - \frac{A_s}{\| A_s \|_2 }\right|\right|_2
\end{equation}

Here, $A_t, A_s$ are the attention maps of the teacher and student respectively.~\citet{zagoruyko2016paying} note that this choice of loss function is especially crucial. 

\subsubsection{Incorporating Jacobian loss} As the activation maps have large spatial dimensions, it is computationally costly to compute the full Jacobians. We hence resort to computing approximating Jacobians in the same manner as previously discussed. In this case, this leads to picking the pixel in the attention map with the largest magnitude, and computing the Jacobian of this quantity w.r.t. the input. We compute the index $(i,j)$ of this maximum-valued pixel for the teacher network and use the same index to compute the student's Jacobian. We then use a loss function similar to Equation~\ref{eqn:matchact}, given by
\begin{equation}
\label{eqn:match_jacobian}
\mathrm{Match~Jacobians} = \left|\left| \frac{\nabla_x f(\X)}{\| \nabla_x f(\X) \|_2} - \frac{\nabla_x g(\X)}{\| \nabla_xg(\X) \|_2 }\right|\right| _2^2
\end{equation}

\subsubsection{Justification for Jacobian loss} We can show that the above loss term corresponds to adding a noise term $\XI_f \propto \| \nabla_x f(\X) \|^{-1}_2 $ for $f(\X)$ and $\XI_g \propto \| \nabla_x g(\X) \|^{-1}_2 $ for $g(\X)$ for the distillation loss. From the first order Taylor series expansion, we see that $g(x + \XI) = g(x) + \XI_g \nabla_x g(\X)$. Thus for networks $f(\cdot)$ and $g(\cdot)$  with different Jacobian magnitudes, we expect different responses for the same noisy inputs. Specifically, we see that $\mathbb{E}_{\XI_g} \| g(x + \XI_g) - g(x) \|^2_2 = \sigma_g^2 \| \nabla_x g(\X) \|^2_2 = \sigma^2 \frac{\| \nabla_x g(\X) \|^2_2}{\| \nabla_x g(\X) \|^2_2} = \sigma^2 $ for a gaussian model with covariance matrix being $\sigma$ times the identity.

\section{Experiments}
We perform three experiments to show the effectiveness of using Jacobians. First, we perform distillation in a limited data setting on the CIFAR100 dataset~\citep{krizhevsky2009learning}. Second, we show on that same dataset that penalizing Jacobian norm increases stability of networks to random noise. Finally, we perform transfer learning experiments on the MIT Scenes dataset~\citep{quattoni2009recognizing}.

\subsection{Distillation}
For the distillation experiments, we use VGG-like~\citep{simonyan2014very} architectures with batch normalization. The main difference is that we keep only the convolutional layers and have only one fully connected layer rather than three. Our workflow is as follows. First, a 9-layer ``teacher'' network is trained on the full CIFAR100 dataset. Then, a smaller 4-layer ``student'' network is trained, but this time on small subsets rather than the full dataset. As the teacher is trained on much more data than the student, we expect distillation to improve the student's performance.

A practical scenario where this would be useful is the case of architecture search and ensemble training, where we require to train many candidate neural network architectures on the same task. Distillation methods can help speed up such methods by using already trained networks to accelerate training of newer models. One way to achieve acceleration is by using less data to train the student dataset.

We compare our methods against the following baselines. \textbf{(1): Cross-Entropy (CE) training} -- Here we train the student using only the ground truth (hard labels) available with the dataset without invoking the teacher network. \textbf{(2): CE + match activations} -- This is the classical form of distillation, where the activations of the teacher network are matched with that of the student. This is weighted with the cross-entropy term which uses ground truth targets. \textbf{(3): Match activations only} -- Same as above, except that the cross-entropy term is not used in the loss function.

We compare these methods by appending the Jacobian matching term in each case. In all cases, we use the squared-error distillation loss shown in Proposition \ref{prop1}. We found that squared loss worked much better than the cross-entropy loss for distillation and it was much easier to tune. 

From Table \ref{C100table} we can conclude that (1) it is generally beneficial to do any form of distillation to improve performance, (2) matching Jacobians along with activations outperforms matching only activations in low-data settings, (3) not matching Jacobians is often beneficial for large data.

One interesting phenomenon we observe is that having a cross-entropy (CE) error term is often not crucial to maintain good performance. It is only slightly worse than using ground truth labels.

For Table \ref{C100table}, we see that when training with activations, Jacobians and regular cross-entropy training (fourth row), we reach an accuracy of $52.43\%$ when training with 100 data points per class. We observe that the overall accuracy of raw training using the full dataset is $54.28\%$. Thus we are able to reach close to the full training accuracy using $1/5^{th}$ of the data requirement.

\begin{table*}
\caption{Distillation performance on the CIFAR100 dataset. Table shows test accuracy (\%). We find that matching both activations and Jacobians along with cross-entropy error performs the best for limited-data settings. The student network is VGG-4 while the teacher is a VGG-9 network which achieves $64.78 \% $ accuracy.}

\center
  \begin{tabular}{ccccccc}
  \hline \rule{0pt}{2.3ex} 
    \textbf{\# of Data points per class $\rightarrow$ } & 1 & 5 & 10 & 50 & 100 & 500 (full) \\ \hline \rule{0pt}{2.3ex} 
    \textbf{Cross-Entropy (CE) training} & 5.69 & 13.9 & 20.03 & 37.6 & 44.92 & 54.28 \\ \hline \rule{0pt}{2ex} 
    \textbf{CE + match activations} & 12.13 & 26.97 & 33.92 & 46.47 & 50.92 & \textbf{56.65} \\ 
    \textbf{CE + match Jacobians} & 6.78 & 23.94 & 32.03 & 45.71 & 51.47 & 53.44  \\ 
    \textbf{CE + match \{activations + Jacobians\}} & \textbf{13.78} & \textbf{33.39} & \textbf{39.55} & \textbf{49.49} & \textbf{52.43} & 54.57 \\ \hline \rule{0pt}{2ex} 
    \textbf{Match activations only} & 10.73 & 28.56 & 33.6 & 45.73 & 50.15 & 56.59 \\ 
    \textbf{Match \{activations + Jacobians\}} & 13.09 & 33.31 & 38.16 & 47.79 & 50.06 & 51.33 \\
\hline
  \end{tabular}
  \label{C100table}
\end{table*}

\subsection{Noise robustness}
We perform experiments where we penalize the Jacobian norm to improve robustness of models to random noise. We train 9-layer VGG networks on CIFAR100 with varying amount of the regularization strength ($\lambda$), and measure their classification accuracy in presence of noise added to the normalized images. From Table \ref{noisetable} we find that using higher regularization strengths is better for robustness, even when the initial accuracy at the zero-noise case is lower. This aligns remarkably well with theory. Surprisingly, we find that popular regularizers such as $\ell_2$ regularization and dropout~\citep{srivastava2014dropout} are not robust.

\begin{table*}[!h]
\caption{Robustness of various VGG-9 models to gaussian noise added to CIFAR100 images at test time. Table shows test accuracy (\%). $\lambda$ is the regularization strength of the Jacobian-norm penalty regularizer. $\gamma$ is the $\ell_2$ regularization strength and $p$ is the dropout value. We see that the Jacobian-norm penalty can be remarkably robust to noise, unlike $\ell_2$ regularization and dropout. }
\center
  \begin{tabular}{cccccc}
  \hline \rule{0pt}{2.3ex} 
    \textbf{Noise std. dev. $\rightarrow$} & 0 & 0.1 & 0.2 & 0.3 & 0.4 \\ \hline 
    \textbf{$\lambda = 0$} & 64.78 & $61.9 \pm 0.07$ & $47.53 \pm 0.23$ & $29.63 \pm 0.16$ & $17.69 \pm 0.17$ \\ \hline \rule{0pt}{2.3ex} 
    \textbf{$\lambda = 0.1$} & {65.62} & $\bm{63.36 \pm 0.18}$ & $53.57 \pm 0.23$ & $37.38 \pm 0.18$ & $23.99 \pm 0.19$ \\  
    \textbf{$\lambda = 1.0$} & 63.59 & $62.66 \pm 0.13$ & ${57.53 \pm 0.17}$ & ${47.48 \pm 0.14}$ & ${35.43 \pm 0.11}$ \\ 
    \textbf{$\lambda = 10.0$} & 61.37 & $60.78 \pm 0.05$ & $\bm{58.28 \pm 0.13}$ & $\bm{52.82 \pm 0.10}$ & $\bm{44.96 \pm 0.19}$\\  \hline \rule{0pt}{2.3ex} 
    $\ell_2$ regularization ($\gamma = 1e-3$) & \textbf{66.92} & $60.41 \pm 0.27$ & $39.93 \pm 0.28$ & $23.32 \pm 0.25$ & $13.76 \pm 0.16$ \\ 
    Dropout $(p = 0.25)$ & 66.8 & $61.53 \pm 0.10$ & $44.34 \pm 0.19$ & $26.70 \pm 0.24$ & $15.77 \pm 0.11$ \\ \hline
  \end{tabular}
  \label{noisetable}
\end{table*}

\subsection{Transfer Learning}
For transfer learning, our objective is to improve training on the target dataset (MIT Scenes) by using Imagenet pre-trained models. Crucially, we want our MIT Scenes model to have a different architecture than the Imagenet model. The teacher model we use is a ResNet-34~\citep{he2016deep} pre-trained on Imagenet, while the student model is an untrained VGG-9 model with batch normalization. We choose VGG-9 because its architecture is fundamentally different from a ResNet. In principle we could use any architecture for the student. We modify this VGG-9 architecture such that it has two sets of outputs, one sharing the label space with Imagenet (1000 classes), and another with MIT Scenes (67 classes, $\sim$ 6k images). The pre-final layer is common to both outputs.

Our baselines are the following. \textbf{(1): Cross-Entropy (CE) training of student with ground truth} -- Here we ignore the VGG-9 branch with 1000 classes and optimize the cross-entropy error on MIT Scenes data. The loss function on this branch is always the same for all methods. \textbf{(2): CE on pre-trained network} -- This is exactly the same as the first baseline, except that the VGG-9 model is initialized from Imagenet pre-trained weights. We consider this as our ``oracle'' method and strive to match it's performance. \textbf{(3): CE + match activations (LwF)} -- This corresponds to the method of~\citet{li2016learning}, where the Imagenet branch output activations of the student are matched with those of the teacher. \textbf{(4): CE + match \{ activations + attention\} } -- This corresponds to the method of~\citet{zagoruyko2016paying}, where attention maps are matched between some intermediate layers.

We add our Jacobian matching terms to the baselines 3 and 4. We provide our results in Table \ref{table:MIT}. In all cases, we vary the number of images per class on MIT Scenes to observe the performance on low-data settings as well. In this case we average our results over two runs by choosing different random subsets.

\begin{table*}[!htbp]
\caption{Transfer Learning from Imagenet to MIT Scenes dataset. Table shows test accuracy (\%). The student network (VGG9) is trained from scratch unless otherwise mentioned. The teacher network used is a pre-trained ResNet34. Results are averaged over two runs.}


\center
\vspace*{-10pt}
  \begin{tabular}{cccccc}
  \hline \rule{0pt}{2.3ex} 
    \textbf{\# of Data points per class $\rightarrow$} & 5 & 10 & 25 & 50 & Full \\ \hline \rule{0pt}{2.3ex} 
    \textbf{Cross-Entropy (CE) training on untrained student network} & 11.64 & 20.30 & 35.19 & 46.38 & 59.33\\ 
    \textbf{CE on pre-trained student network (Oracle)} & \textbf{25.93} & \textbf{43.81} & \textbf{57.65} & \textbf{64.18} & \textbf{71.42} \\ \hline \rule{0pt}{2.3ex} 
    \textbf{CE + match activations~\citep{li2016learning}} & 17.08 & 27.13 & 45.08 & 55.22 &  65.22  \\ 
    \textbf{CE + match \{activations + Jacobians\} } & 17.88 & 28.25 & 45.26 & 56.49 & 66.04  \\ \hline \rule{0pt}{2.3ex} 
    \textbf{CE + match \{activations + attention\}~\citep{zagoruyko2016paying} }  & 16.53 & 28.35 & 46.01 & 57.80 & \textbf{67.24}  \\ 
    \textbf{CE + match \{activations + attention + Jacobians\}} & \textbf{18.02} & \textbf{29.25} & \textbf{47.31} & \textbf{58.35} & \textbf{67.31}  \\ 
    \hline
  \end{tabular}
  \label{table:MIT}
\end{table*}

\begin{table}[h]
\caption{Ablation experiments over choice of feature matching depth. ($\mathcal{T}$, $\mathcal{S}$) denotes teacher (ResNet34) and student (VGG9) feature depths. These pairs are chosen such that resulting features have same spatial dimensions. We see that matching the shallowest feature works the best. Results are averaged over two runs.}
\center
\vspace{-10pt}
  \begin{tabular}{c|cccc}
  \hline \rule{0pt}{3ex}
    \specialcell{\textbf{Feature matching} \vspace*{-2pt}\\ \textbf{depth ($\mathcal{T}$, $\mathcal{S}$)}} & (7, 2) & (15, 4) & (27, 6) & (33, 8) \\\hline \rule{0pt}{2.5ex}
    \textbf{Accuracy (\%)} & \textbf{22.39} & 21.98 & 20.45 & 20.03 \\ \rule{0pt}{3ex}
    \specialcell{\textbf{Jacobian loss}\vspace*{-2.5pt} \\\textbf{reduction (\%)}} & \textbf{25.88} & 15.59 & 11.55 & 1.25 \\ 
    \hline
  \end{tabular}
 \vspace{-10pt}
  \label{table:MIT_ablations_feat_lvl}
 \end{table}

\begin{table}[]

\caption{Ablation experiments over the computation of Jacobian. Here, $s$ is the size of the attention map. ``Full'' is global average pooling, and ``None'' is no average pooling. We see that using average pooling while computing Jacobians helps performance. Results are averaged over two runs.}

\center

\vspace{-2ex}

  \begin{tabular}{c|ccccc}
    \hline \rule{0pt}{3ex}
    \specialcell{\textbf{Average Pool} \vspace*{-2.5pt}\\ \textbf{Window size}} & Full & $s/3$ & $s/5$ & $s/7$ & \specialcell{None}\\ \hline \rule{0pt}{2.5ex}
    \textbf{Accuracy (\%)} & 20.00 & 21.20 & \textbf{21.87} & 21.09 & 19.74 \\ 
    \hline
  \end{tabular}

  \label{table:MIT_ablations_winsize}
\end{table}

Experiments show that matching activations and attention maps increases performance at all levels of data size. It also shows that Jacobians improve performance of all these methods. However, we observe that none of the methods match the oracle performance of using a pre-trained model. The gap in performance is especially large at intermediate data sizes of $10$ and $25$ images per class. 

\subsubsection{Which features to match?} An important practical question is the choice of intermediate features to compute attention maps for matching. The recipe followed by~\citet{zagoruyko2016paying} for ResNets is to consider features at the end of a residual block. \footnote{A residual block is the set of all layers in between two consecutive max-pooling layers in a ResNet. All features in a block have the same dimensions.} As there are typically 3-5 max-pooling layers in most modern networks, we have 3-5 intermediate features to match between any typical teacher and student network. Note that we require the attention maps (channelwise aggregated features) to be of similar spatial size to match.~\citet{zagoruyko2016paying} match at all such possible locations, and we use the same approach.

However, computing Jacobians at all such locations is computationally demanding and perhaps unnecessary. We observe that if we compute Jacobians at later layers, we are still not able to reduce training Jacobian loss, possibly due to a ``second-order'' vanishing gradient effect. At suitable intermediate layers, we see that loss reduction occurs. This is reflected in Table \ref{table:MIT_ablations_feat_lvl}, where we vary the feature matching depth and observe performance. We observe that the Jacobian loss reduction (during training) is highest for the shallowest layers, and this corresponds to good test performance as well. These ablation experiments are done on the MIT Scenes dataset picking only $10$ points per class.

\subsubsection{Feature Pooling to compute Jacobians} Instead of considering a single pixel per attention map to compute Jacobians, we can aggregate a large number of large-magnitude pixels. One way to do this is by average pooling over the attention map, and then picking the maximum pixel over the average pooled map. In Table \ref{table:MIT_ablations_winsize} we vary the window size for average pooling and observe performance. We observe that it is beneficial to do average pooling, we find that using a window size of $\mathrm{(feature~size)} / 5$ works the best. These ablation experiments are done on the MIT Scenes dataset picking only $10$ points per class.

\section{Conclusion}
In this paper we considered matching Jacobians of deep neural networks for knowledge transfer. Viewing Jacobian matching as a form of data augmentation with gaussian noise motivates their usage, and also informs us about the loss functions to use. We also connected a recent transfer learning method (LwF) to distillation, enabling us to incorporate Jacobian matching.

Despite our advances, there is still a large gap between distillation-based methods and the oracle method of using pre-trained nets for transfer learning. Future work can focus on closing this gap by considering more structured forms of data augmentation than simple noise addition.

\bibliography{icml_bib}
\bibliographystyle{icml2018}

\clearpage

\part*{Supplementary Material}

\section{Proof for Proposition 1}

\begin{prop_star}
	Consider the squared error cost function for matching soft targets of two neural networks with $k$-length targets ($ \in \mathbb{R}^k$), given by $\ell(\mathcal{T}(\X), \mathcal{S}(\X)) = \sum_{i=1}^{k}(\mathcal{T}^i(\X) - \mathcal{S}^i(\X))^2$, where $\X \in \mathbb{R}^D$ is an input data point. Let $\XI ~(\in \mathbb{R}^D)= \sigma~\mathbf{z}$ be a scaled version of a unit normal random variable $\mathbf{z}~ \in \mathbb{R}^D$ with scaling factor $\sigma \in \mathbb{R}$. Then the following is locally true.
	
	\begin{align*}
	&\mathbb{E}_{\XI}\left[\sum_{i=1}^{k}\left(\mathcal{T}^i(\X + \XI) - \mathcal{S}^i(\X + \XI)\right)^2\right] \\ & = \sum_{i=1}^{k}\left(\mathcal{T}^i(\X) - \mathcal{S}^i(\X)\right)^2 + \sigma^2 \sum_{i=1}^{k} \|\nabla_x\mathcal{T}^i(\X) - \nabla_x\mathcal{S}^i(\X)\|_2^2 \\&+ \mathcal{O}(\sigma^4)
	\end{align*}

\end{prop_star}

\begin{proof} There exists $\sigma$ and $\XI$ small enough that first-order Taylor series expansion holds true

\begin{align}
& \mathbb{E}_{\XI}\left[ \sum_{i=1}^{k}\left(\mathcal{T}^i(\X + \XI) - \mathcal{S}^i(\X + \XI)\right)^2\right] \nonumber \\
= &~ \mathbb{E}_{\XI}\left[ \sum_{i=1}^{k}\left(\mathcal{T}^i(\X) + \XI^T \nabla_x \mathcal{T}^i(\X) - \mathcal{S}^i(\X) - \XI^T \nabla_x \mathcal{S}^i(\X)\right)^2\right] \nonumber
\\+ & \mathcal{O}(\sigma^4)\nonumber \\
= & \sum_{i=1}^{k}\left(\mathcal{T}^i(\X) - \mathcal{S}^i(\X)\right)^2  \nonumber
\\+ & \mathbb{E}_{\XI}\left[ \sum_{i=1}^{k}  \left[ \XI^T \left( \nabla_x \mathcal{T}^i(\X) - \nabla_x \mathcal{S}^i(\X)\right) \right]^2\right] + \mathcal{O}(\sigma^4) \label{seclineprop2}
\end{align}

To get equation \ref{seclineprop2}, we use the fact that mean of $\XI$ is zero. To complete the proof, we use the diagonal assumption on the covariance matrix of $\XI$.

\end{proof}

Proofs of other statements are similar. For proof for cross-entropy loss of Proposition 2, use a second order Taylor series expansion of $\log(\cdot)$ in the first step.

\section{Proof for Proposition 3}

\begin{prop_star} 
From the notations in the main text, we have
\begin{align*}
\frac{1}{|\mathcal{D}_l|} \sum_{\X \sim \mathcal{D}_l} \ell (f(\X), g(\X)) &\le  \max_{\X \sim \mathcal{D}_s} \ell (f(\X), g(\X)) \\ &+ \mathrm{K} \mathcal{H}_a(\mathcal{D}_l, \mathcal{D}_s) 
\end{align*}
\end{prop_star}

\begin{proof} Let us denote $\rho(\X) = \ell(f(\X), g(\X))$ for convenience. Assume Lipschitz continuity for $\rho(\X)$ with Lipschitz constant $\mathrm{K}$, and distance metric $\psi_{\X}(\cdot, \cdot)$ in the input space - 
\vspace{-0.5pt}
\begin{align*}
\| \rho(\X_1) - \rho(\X_2) \|_1 \leq \mathrm{K} \psi_{\X}(\X_1, \X_2)  \\
\implies \rho(\X_1) \le \rho(\X_2) + \mathrm{K} \psi_{\X}(\X_1, \X_2)
\end{align*}
Assuming that $\rho(\X_1) \ge \rho(\X_2)$. Note that it holds even otherwise, but is trivial.

Now, for every datapoint $\X_l \in \mathcal{D}_l$, there exists a point $\X_s \in \mathcal{D}_s$ such that $\psi_{\X}(\X_s, \X_l)$ is the smallest among all points in $\mathcal{D}_s$. In other words, we look at the point in $\mathcal{D}_s$ closest to each point $\X_l$. Note that in this process only a subset of points $\mathrm{d}_s \subseteq \mathcal{D}_s$ are chosen, and individual points can be chosen multiple times. For these points, we can write 

\begin{align*}
\rho(\X_l) \leq \rho(\X_s) + \mathrm{K} \psi_{\X}(\X_s, \X_l) &\\
\implies \frac{1}{|\mathcal{D}_l|} \sum_{\X_l \sim \mathcal{D}_l} \rho(\X_l) \leq \frac{1}{|\mathcal{D}_l|} \sum_{\X_s \mathrm{ ~closest~ to~ }\X_l} &\rho(\X_s) \\ + \frac{1}{|\mathcal{D}_l|} \sum_{\X_s \mathrm{ ~closest~ to~ }\X_l} &\mathrm{K} \psi_{\X}(\X_s, \X_l)
\end{align*}

We see that $\frac{1}{|\mathcal{D}_l|} \sum_{\X_s} \rho(\X_s) \leq \max_{\X \sim \mathrm{d}_s} \rho(\X) \leq \max_{\X \sim \mathcal{D}_s} \rho(\X)$, which is a consequence of the fact that the max is greater than any convex combination of elements.  

Also, we have $\psi_{\X}(\X_l, \X_s) \leq \mathcal{H}_a(\mathcal{D}_l, \mathcal{D}_s)$, which is the maximum distance between any two `closest' points from $\mathcal{D}_l$ to $\mathcal{D}_s$ (by definition). 

Applying these bounds, we have the final result.

\end{proof}

\subsection{Proof for Corollary}
\begin{corollary}
For any superset $\mathrm{D'}_s \supseteq \mathcal{D}_s$ of the target dataset, $\mathcal{H}_a(\mathcal{D}_l, \mathcal{D'}_s) \leq \mathcal{H}_a(\mathcal{D}_l, \mathcal{D}_s)$
\end{corollary}

\begin{proof}
From the previous proof, we have $\rho(\X_l) \leq \rho(\X_s) + \mathrm{K} \psi_{\X}(\X_s, \X_l)$ for an individual point $\X_l$. Now if we have $\mathrm{D'}_s \supseteq \mathcal{D}_s$, then we have $\rho(\X_l) \leq \rho(\X'_s) + \mathrm{K} \psi_{\X}(\X'_s, \X_l)$, where $\X'_s$ is the new point closest to $\X_l$. It is clear that $\psi_{\X}(\X'_s, \X_l) \leq \psi_{\X}(\X_s, \X_l)$ for all $\X_l$. Hence it follows that $\mathcal{H}_a(\mathcal{D}_l, \mathcal{D'}_s) \leq \mathcal{H}_a(\mathcal{D}_l, \mathcal{D}_s)$.

\end{proof}

\section{Experimental details}

\subsection{VGG Network Architectures}
The architecture for our networks follow the VGG design philosophy. Specifically, we have blocks with the following elements:

\begin{itemize}
\item $3 \times 3$ conv kernels with $c$ channels of stride 1
\item Batch Normalization
\item ReLU
\end{itemize}

Whenever we use Max-pooling (M), we use stride $2$ and window size $2$.

The architecture for VGG-9 is - $[ 64 - M - 128- M-  256 - 256 - M - 512 - 512 - M - 512 - 512 - M  ]$. Here, the number stands for the number of convolution channels, and $M$ represents max-pooling. At the end of all the convolutional and max-pooling layers, we have a Global Average Pooling (GAP) layer, after which we have a fully connected layer leading up to the final classes. Similar architecture is used for the case of both CIFAR and MIT Scene experiments.

The architecture for VGG-4 is - $[64 - M - 128 -M - 512 - M ]$.

\subsection{Loss function}
The loss function for distillation experiments use the following form.

\begin{equation*}
\ell(\mathcal{S}, \mathcal{T}) = \alpha \times \mathrm{(CE)} + \beta \times \mathrm{(Match~ Activations)} + \gamma \times \mathrm{(Match~Jacobians)}
\end{equation*}

In our experiments, $\alpha, \beta, \gamma$ are either set to $1$ or $0$. In other words, all regularization constants are $1$.

Here, `CE' refers to cross-entropy with ground truth labels. `Match Activations' refers to squared error term over pre-softmax activations of the form $(y_s - y_t)^2$. `Match Jacobians' refers to the same squared error term, but for Jacobians.

For the MIT Scene experiments, $\alpha, \beta, \gamma$ are either set to $10$ or $0$, depending on the specific method. To compute the Jacobian, we use average pooling over a $feature~size/5$ window with a stride of $1$. We match the Jacobian after the first residual block for resnet, and after the second max-pool for VGG. This corresponds to feature level ``1'' in the ablation experiments.

\subsection{Optimization}
For CIFAR100 experiments, we run optimization for $500$ epochs. We use the Adam optimizer, with an initial learning rate of $1e-3$, and a single learning rate annealing (to $1e-4$) at $400$ epochs. We used a batch size of $128$.

For MIT Scenes, we used SGD with momentum of $0.9$, for $75$ epochs. The initial learning rate is $1e-3$, and it is reduced $10$ times after $40$ and $60$ epochs. We used batch size $8$. This is because the Jacobian computation is very memory intensive.

\checknbdrafts

\end{document}